\documentclass[review,authoryear]{elsarticle}
\usepackage{amssymb}
\usepackage{amsmath}
\usepackage{graphicx}
\usepackage{lineno}
\usepackage{amsthm}
\usepackage{algorithm}
\usepackage{algpseudocode}

\algnewcommand\algorithmicinput{\textbf{Input:}}
\algnewcommand\Input{\item[\algorithmicinput]}
\algnewcommand\algorithmicoutput{\textbf{Output:}}
\algnewcommand\Output{\item[\algorithmicoutput]}

\newtheorem{theorem}{Theorem}
\newtheorem{lemma}[theorem]{Lemma}
\newtheorem{definition}{Definition}

\begin{document}
	
\ifpreprint
\setcounter{page}{1}
\else
\setcounter{page}{1}
\fi

\begin{frontmatter}
			
\title{A cellular automata approach to local patterns for texture recognition}

\author[imecc]{Joao B. Florindo\corref{cor1}} 
\cortext[cor1]{Corresponding author}
\ead{jbflorindo@ime.unicamp.br}
\author[fcm]{Konradin Metze} 
\ead{kmetze@fcm.unicamp.br}

\address[imecc]{Institute of Mathematics, Statistics and Scientific Computing - University of Campinas\\
	Rua S\'{e}rgio Buarque de Holanda, 651, Cidade Universit\'{a}ria "Zeferino Vaz" - Distr. Bar\~{a}o Geraldo, CEP 13083-859, Campinas, SP, Brasil}
\address[fcm]{Faculty of Medical Sciences, State University of Campinas (UNICAMP), Campinas, Brazil}

\begin{abstract}
Texture recognition is one of the most important tasks in computer vision and, despite the recent success of learning-based approaches, there is still need for model-based solutions. This is especially the case when the amount of data available for training is not sufficiently large, a common situation in several applied areas, or when computational resources are limited. In this context, here we propose a method for texture descriptors that combines the representation power of complex objects by cellular automata with the known effectiveness of local descriptors in texture analysis. The method formulates a new transition function for the automaton inspired on local binary descriptors. It counterba\-lances the new state of each cell with the previous state, in this way introducing an idea of ``controlled deterministic chaos''. The descriptors are obtained from the distribution of cell states. The proposed descriptors are applied to the classification of texture images both on benchmark data sets and a real-world problem, i.e., that of identifying plant species based on the texture of their leaf surfaces. Our proposal outperforms other classical and state-of-the-art approaches, especially in the real-world problem, thus revealing its potential to be applied in numerous practical tasks involving texture recognition at some stage.
\end{abstract}


\end{frontmatter}

\section{Introduction}

Since its popularization in the sixties, chaos theory, in its modern sense, has attracted attention in numerous areas. Models for image processing and analysis have also been proposed, benefiting from tools originally developed for the analysis of chaotic systems \citep{GJ14,R10,GWJWZ10,SWBB15,Y17}.

Among the classical models of deterministic chaos, cellular automata (CA) \citep{W02} have been natural candidates to model digital images, mainly due to their intrinsic representation over a two-dimensional grid, which allows the association between pixel intensity and cell states. 

Whereas the literature has presented applications of CA models in image processing \citep{WT07,GJ14,R06,R10,LEA10,GS18} or general applications of pattern recognition \citep{CI06,GWJWZ10}, the use of CAs to provide image descriptors is still a topic little explored, at least in explicit terms.

Cellular automata and texture descriptors share a fundamental property: the locality. Well established texture descriptors such as Haralick features \citep{H79}, local binary patterns \citep{OPM02}, or bag of visual words \citep{VZ09}, rely on the idea of a local descriptor capable of quantifying the relation among pixels within a neighborhood. In a similar way, CAs evolve in time by applying predefined rules that essentially depend on the neighborhood states. Nevertheless, CAs add en extra ingredient to the process that is the successive application of a non-linear operation. This approach has also recently demonstrated its powerfulness in texture recognition \citep{BM13,F20}, providing a model competitive even with the state-of-the-art learning-based methods. The utility of CA models as a robust texture descriptor has also been recently verified in \citep{SWBB15} and confirmed the expectations.

Therefore we propose the theoretical development and practical application of a texture descriptor that combines the pattern recognition abilities of CAs with the advantages of a texture descriptor based on the local binary patterns (LBP) theory \citep{OPM02}. The proposed descriptors are named Cellular Automata Texture descriptors (CATex).

The proposed methodology consists in successive applications of an LBP-based operator exploring its nonlinear characteristic. The simple successive application yields, however, a rapid increase in the global complexity of the model, as expected from chaos theory. These changes disorganizes excessively  the pixel patterns of the original image and in this way hampers recognition tasks. CATex descriptors circumvent this problem by introducing a control parameter that combines the cell state in the next iteration with the state in the current iteration, thus ensuring that the ``chaoticization'' process takes place in a controlled manner.

The proposed CATex descriptors are tested on texture classification tasks, being compared both with classical texture features, like LBP \citep{OPM02} and VZ-Joint \citep{VZ09}, and with modern approaches, such as binarized statistical image features (BSIF) \citep{KR12}, scale invariant feature transform (SIFT) combined with bag-of-visual-words (BOVW) \citep{CMKMV14}, and convolutional neural networks (CNN) \citep{CMKV16}. Benchmark databases like UIUC \citep{LSP05}, UMD \citep{XJF09} and KTH-TIPS2b \citep{HCFE04} are used for comparison. The developed methodology is also applied to a ``real-world'' problem, namely, the identification of plant species based on images of the leaf surface. The classification accuracy achieved by the proposed model attests its value as an alternative for texture recognition in general. In particular, the new method will be helpful in situations where learning-based approaches are not advantageous, for instance, when there are only few data for training, which is a common problem in several areas, for example, in medical applications.  

\section{Related works}

Applications of chaos theory to image description are proposed, for example, in \citep{Y17}, where the authors develop an analytical model for the image function employing a Bezier parametric representation of the image surface and analyze complex chaotic patterns arising after successive iterations of such image function. They also provide a theoretical framework where they postulate a relation between the role that chaotic phenomena play in the brain and the interpretation of digital images.

Concepts of chaos theory have also been used for the solution of other problems related to pattern recognition, for example, in classifiers and evolutionary algorithms that can be applied to image recognition \citep{CI06,GWJWZ10}.

Cellular automata, in particular, have been used for well known tasks in image processing, like edge and spot detection and noise filtering in \citep{WT07} and the segmentation of medical images in \citep{GJ14}. Another approach based on CAs for image processing has been developed by Rosin \citep{R10}, which was inspired by works like \citep{R06}. While the  latter was focused on gray level images, the former was applied to binary  images. In both cases, a sequential floating forward search mechanism was adopted to select CA rules capable of doing basic processing tasks such as noise filtering, thinning, etc. Combinations of CA with Markov models have also been presented in \citep{LWML19} for an application in remote sensing. Cellular automata have also been used to simulate physical models representing material images \citep{GS18}. In \citep{EPAIW15}, a CA with a contextual-based transition function is employed for the classification of satellite images. Basic models similar to the classical ``game of life'' have been applied for the detection of micro-calsification in mamographies \citep{HSH13}. Image analysis and CA can also be combined in the opposite way, i.e., using image analysis to classify CA patterns \citep{SBOBB16}. More recently, a cellular neural network, which is structurally inspired in cellular automata, has been successfully employed for texture recognition in \citep{JCSZ20}.

Regarding all these investigations, it is surprising that we can hardly find CA-based approaches combining image processing and pattern recognition for image classification. The literature has presented solutions based on binarization of the image followed by modeling of a CA with classical transition functions \citep{LEA10}. The study done by Silva et al. Silva et al. \citep{SWBB15} demonstrates how a CA model can be helpful for the recognition of images, especially textures. They developed a nonlinear operator based on a material corrosion model and applied it successively over the image, collecting information from the statistical distribution of the transformed image after each iteration. 

\section{Proposed Methodology}

\subsection{Cellular automata}

Cellular automata (CA) are dynamical models defined over a discrete tesselation (usually a grid) and evolving at discrete time steps. Their main elements are formally presented in the following definition.
\begin{definition}
	A cellular automaton is represented by a sextuple $<G,S,s,s_0,N,\Phi>$, where
	\begin{itemize}
		\item $G$ is a two-dimensional grid of cells $c_i$ ($i \in \mathbb{N}$);
		\item $S$ is a finite set of states ($S \subset \mathbb{N}$);
		\item $s(c_i,t)$ is the function that provides the state of each cell $c_i$ at time $t$; 
		\item $s_0 (c_i)$ defines the initial state of each $c_i$;
		\item $N(c_i)$ is the neighborhood function, which associates each cell $c_i$ to its neighborhood (notice that neighborhood here is an abstract predefined concept and not necessarily implies spatial proximity);
		\item $\Phi(c_i,N,t)$ is the transition function, which receives as input the current state of cell $c_i$ and the states of its neighbor cells at time $t$, and outputs the state of $c_i$ in the next time step $t+1$.
	\end{itemize} 
\end{definition}

\subsection{Local binary patterns}

Local binary patterns (LBP) \citep{OPM02} are texture descriptors that rely on comparing the gray level of each pixel with that of its neighbors. In its most popular and basic version, only the sign of such difference is considered to compose the binary local codes. Here we employ a rotation-invariant version. In this approach, similar local patterns (uniform patterns) are counted as a single pattern, in this way reducing dimensionality and being more discriminative than the basic version, as demonstrated in \citep{OPM02}. Such descriptor is obtained by
\begin{equation}
LBP_{P,R}^{riu2} = \left\{
\begin{array}{ll}
\sum_{p=0}^{P-1}H(g_p-g_c)2^p & \mbox{if } U(LBP_{P,R})\geq 2\\
P+1 & \mbox{otherwise}, 
\end{array}
\right.
\end{equation}
where $g_c$ is the gray level of the reference (central) pixel, $g_p$ are the gray level of each neighbor pixel, $H$ is the Heaviside step function ($H(x)=1$ if $x>0$ and $H(x)=0$ otherwise), $R$ is the radius of the neighborhood and $P$ is the number of neighbor pixels sampled over a circle centered at the reference pixel with radius $R$. Moreover, we have the uniformity function $U$ defined by
\begin{equation}
U(LBP_{P,R}) = |H(g_{P-1}-g_c)-H(g_0-g_c)| + \sum_{p=1}^{P-1}|H(g_{p}-g_c)-H(g_{p-1}-g_c)|.
\end{equation}

\subsection{Proposed method}

Here we propose a cellular automaton model for texture recognition, whose transition function is based on $LBP_{P,R}^{riu2}$ codes. The proposed descriptors are dubbed CATex (for Cellular Automata Texture) descriptors. The elements of the cellular automaton in our model are described in the following.

The grid $G$ in this case is the grid of the image pixels, with coordinates in the space $\mathbb{Z}_{M\times N}$, where $M$ and $N$ are the image dimensions.

The CA is initialized with each cell corresponding to a pixel in the analyzed image and each state corresponding to the gray level of that pixel (usually an integer value between $0$ and $255$).

The neighborhood of each CA cell $c_i$ is determined by the $LBP_{P,R}^{riu2}$ neighborhood, i.e., by $P$ cells sampled over a circle centered at $c_i$ with radius $R$, using bilinear interpolation as in \citep{OPM02}.

The transition function is parameterized by $P$ and $R$ and performs a weighted sum involving the current cell state ($s(c_i)$) and the $LBP_{P,R}^{riu2}$ code of $c_i$:
\begin{equation}
	\Phi_{P,R}(c_i,t) = (1-\alpha)s(c_i,t-1) + \alpha LBP_{P,R}^{riu2} (c_i),
\end{equation}
where $\alpha$ plays a fundamental role in the method, controlling the strength of the action of the transition function. The nonlinearity of $LBP_{P,R}^{riu2}$ tends to rapidly evolve to a chaotic scenario. Parameter $\alpha$ acts as a chaos controller, allowing in this way for the cell states to be used as image descriptors.

The final descriptors are provided by the histograms of the LBP codes computed over the CA map at each CA time step for different combinations of parameters $P$ and $R$ in the LBP algorithm. More specifically, we employ the 9 combinations suggested in \citep{LFGWP17} and also empirically confirmed here as an efficient scheme, i.e., $(P,R) = \{(8,1),(16,2),(24,3),(24,4),(24,5),(24,6),(24,7),(24,8),\\(24,9)\}$. We also found out that 20 iterations of the CA was sufficient to provide robust and effective image descriptors. Finally, considering the high number of descriptors generated by such algorithm, a post-processing step of principal component analysis \citep{J86} was employed to reduce dimensionality.

Algorithm \ref{alg:ca} presents the pseudo-code of the process to generate the CATex descriptors and Table \ref{tab:routines} gives explanation regarding some auxiliary routines employed in Algorithm \ref{alg:ca}. Some points deserve some extra attention here. At line 4, the image is submitted to reflective padding with $R$ rows and columns added on both directions. The objective is to preserve the size of the LBP map, as the output of $LBP_{P,R}^{riu2}$ algorithm removes $R$ rows and columns from the border of the image. As here the LBP mapping is computed iteratively, such reduction would rapidly make the map very small to provide any useful descriptor. Lines 3-7 compute descriptors from the $LBP_{P,R}^{riu2}$ map of the original image and therefore correspond to the original $LBP_{P,R}^{riu2}$ descriptors as presented in \citep{OPM02}. Lines $8-10$ perform the first iteration of the CA, by computing the weighted combination of the original image (initial state of the CA) with the $LBP_{P,R}^{riu2}$ map. Variable $I_{stack}$ stores the updated CA states at each time step for each $R$ value. Lines $12-16$ compute the LBP codes of the current states of the CA and add the histogram of such codes to the descriptor vector. Lines 17-19 updates the state of the CA using again the weighted combination of the current state with the LBP map.

\begin{algorithm}[H]
	\caption{Algorithm of CATex descriptors.}
	\label{alg:ca}
	\begin{algorithmic}[1] 
		
		\Input $I$ (image normalized in $[0,1]$)
		\Output $\mathfrak{D}$ (CATex descriptors)
		\State $P \gets [8 \quad 16 \quad 24 \quad 24 \quad 24 \quad 24 \quad 24 \quad 24 \quad 24]$
		\State $\mathfrak{D} \gets \emptyset$
		\For{$R = 1 \textrm{ to } 9$}
			\State $I_{aux} \gets \textrm{padding}(I,R)$
			\State $L[R] \gets \mathrm{LBP^{riu2}}(I_{aux},P[R],R)$
			\State $\mathfrak{D} \gets \mathfrak{D} \cup \mathrm{histogram}(L[R])$
		\EndFor
		\For{$R = 1 \textrm{ to } 9$}
			\State $I_{stack}[R] \gets (1-\alpha)I + \alpha L[R]$
		\EndFor
		\For{$k = 1 \textrm{ to } 20$}
			\For{$R = 1 \textrm{ to } 9$}
				\State $I_{aux} \gets \textrm{padding}(I_{stack}[R],R)$
				\State $L[R] \gets \mathrm{LBP^{riu2}}(I_{aux},P[R],R)$
				\State $\mathfrak{D} \gets \mathfrak{D} \cup \mathrm{histogram}(L[R])$
			\EndFor
			\For{$R = 1 \textrm{ to } 9$}
				\State $I_{stack}[R] \gets (1-\alpha)I_{stack}[R] + \alpha L[R]$
			\EndFor						
		\EndFor				
		
	\end{algorithmic}
\end{algorithm}

\begin{table}[h]	
	\centering
	\caption{Auxiliary routines used by Algorithm \ref{alg:ca}.}	
	\begin{tabular}{ll}
		\hline
		$\textrm{padding}(I,r)$ & Symmetric padding of image $I$ with mirror reflections of itself and size $r$ \\
		$\textrm{LBP}^{riu2}(I,P,R)$ & $LBP_{P,R}^{riu2}$ map of image $I$ as described in \citep{OPM02} (normalized in $[0,1]$) \\
		$\mathrm{histogram}(L)$ & Histogram of LBP map $L$, counting all possible LBP code values\\
		\hline
	\end{tabular}
	\label{tab:routines}	
\end{table}

Figure \ref{fig:method} visually illustrates the resulting outcomes of the proposed method at two specific iterations of the cellular automata: first and fifth iterations. To avoid an excessively cluttered image, we show only the cases $R=1$ and $R=9$. In the proposal, all integer values between 2 and 8 for $R$ are also used. The ``blurring'' effect in the iterated image is caused by the reduced number of $LBP_{P,R}^{riu2}$ possible codes. Those codes have more and more weight in the iterated map as the CA evolves.
\begin{figure}
	\includegraphics[width=\textwidth]{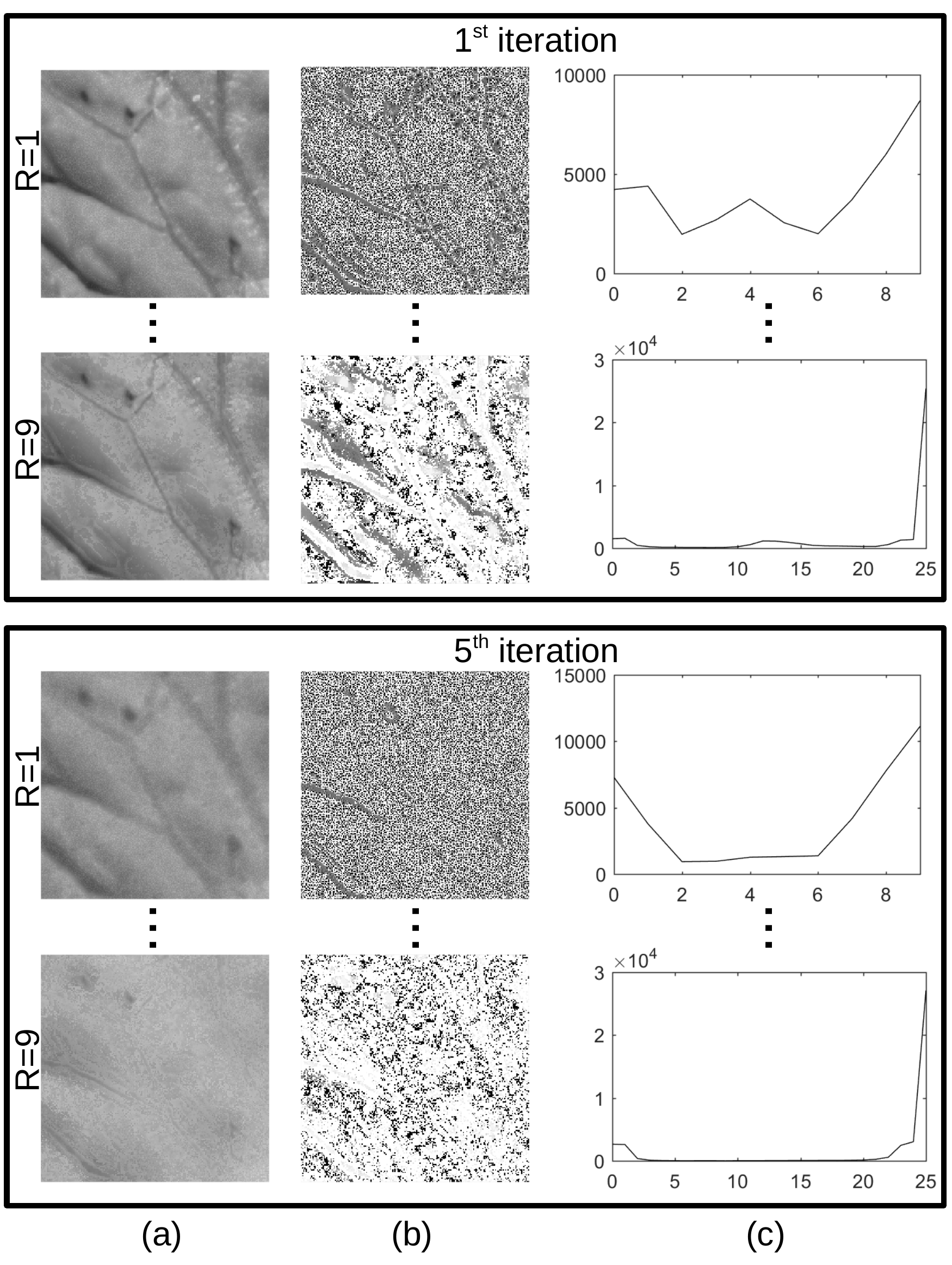}
	\caption{Outcomes generated by iterations 1 and 5 in the proposed method ($\alpha=0.10$). (a) Iterated image. (b) LBP codes. (c) LBP histogram used to compose the final descriptors.}
	\label{fig:method}
\end{figure}

\subsection{Formal analysis}\label{sec:motivation}

LBP encoding can be seen as a non-linear function applied to each pixel neighborhood. We will show such operation more precisely now for the basic LBP method with an 8-neighborhood. 

Given a central pixel value $g_c$ and its 8-neigbors $g_i$ taken in an appropriate direction we have
\begin{equation}
	L(g_c;g_i) = \sum_{i=0}^{7}H(g_i-g_c)2^i,
\end{equation}
where $H$ stands for the Heaviside step function. The value $g_i-g_c$ resembles in many aspects the idea of local derivative. To facilitate the comprehension we reduce the operation to the one-dimensional domain. In this case we may have at each position $x_i$:
\begin{equation}
	L(x_i) = H(x_{i-1}-x_i)2^0 + H(x_{i+1}-x_i)2^1.
\end{equation}
This expression can be easily rewritten in terms of the classical first order finite difference operator $D$, which mimics first derivative in a discrete space. This is classically defined by
\begin{equation}
	D(x_i) = x_{i+1} - x_i
\end{equation}
and therefore $L$ can be expressed as
\begin{equation}
	L(x_i) = H(D(x_{i-1})) - 2H(D(x_i)).
\end{equation}
Figure \ref{fig:diffLBP} illustrates how these operators act on the discrete function and their obvious similarities. 
\begin{figure}
	\centering
	\includegraphics[width=.5\textwidth]{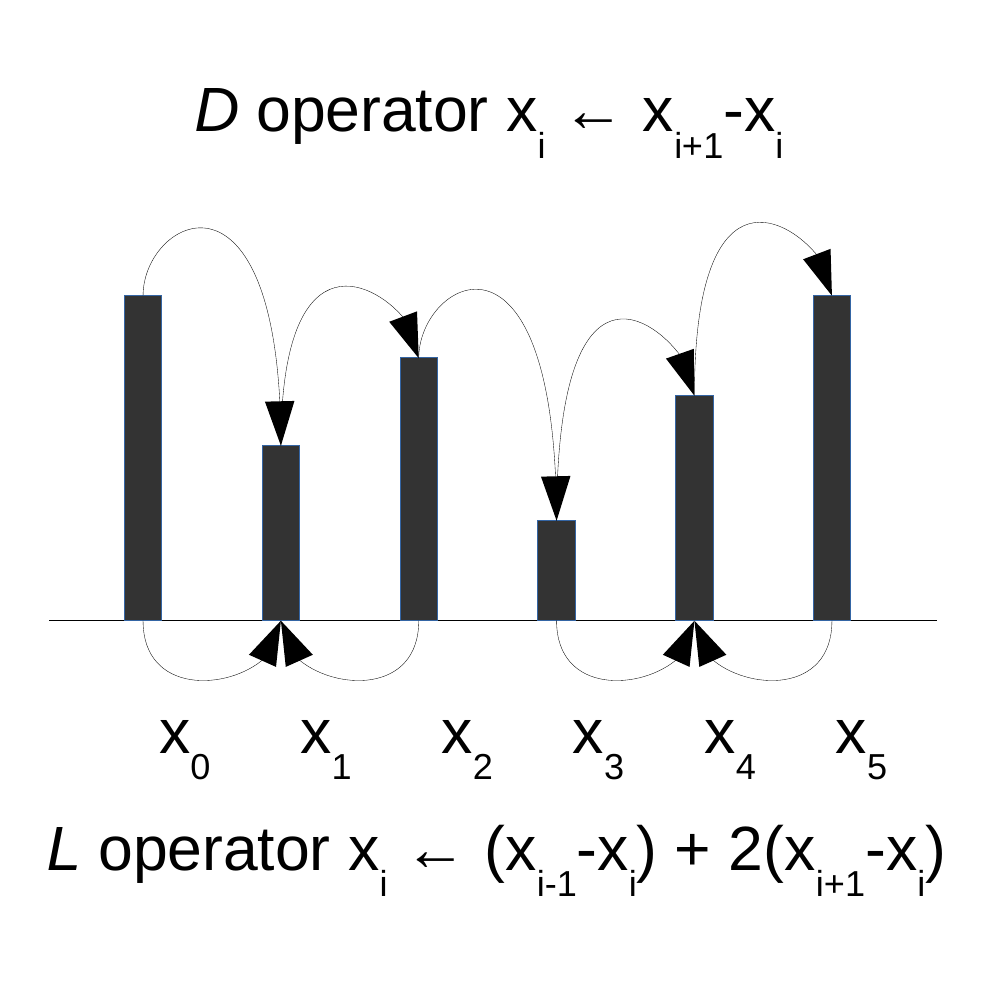}
	\caption{$D$ and $L$ operators. An arrow connecting $x_i$ and $x_j$ means that the difference $x_j-x_i$ is accumulated at the arrow end.}
	\label{fig:diffLBP}
\end{figure}
This notation also provides some facilities to understanding the successive application of $L$:
\begin{equation}
	L(L(x_i)) = H(D(L(x_{i-1}))) - 2H(D(L(x_i))).
\end{equation}
Like continuous derivative, the operator $D$ is also linear, i.e,
\begin{equation}
	D(\alpha x + \beta y) = \alpha D(x) + \beta D(y),
\end{equation}
so that inspecting one part of the difference is sufficient to provide us relevant information. Given that we focus on the term $H(D(L(x_i)))$ and, even more specifically, on $D(L(x_i))$:
\begin{equation}
	D(L(x_i)) = D(H(D(x_{i-1}))) - 2D(H(D(x_i))).
\end{equation}
As the Heaviside step function only assumes values $0$ or $1$, we have three possibilities for $D(H(D(x_i)))$: $1$, $-1$ or $0$. The first two cases imply that the function is not monotonic at that point, i.e., it changes its slope. They are also directly related to $D(D(x_i)) = D^2(x_i)$.
\begin{lemma}
	\begin{equation}
	D(H(D(x_i)))=1 \Rightarrow D^2(x_i)>0 
	\end{equation}
\end{lemma}
\begin{proof}
	The combination resulting in $D(H(D(x_i)))=1$ requires $D(x_{i-1}) < 0$ and $D(x_i) > 0$, which implies 
	\begin{equation}\label{eq:p1}
		x_{i-1}>x_i \qquad \mbox{and} \qquad x_{i+1}>x_i.
	\end{equation}
	At the same time, we have for $D^2(x_i)$:
	\begin{equation}\label{eq:p2}
		D^2(x_i) = D(x_i)-D(x_{i-1}) = (x_{i+1}-x_{i}) - (x_i-x_{i-1}) = x_{i-1}+x_{i+1}-2x_i.
	\end{equation}
	From \ref{eq:p1} we have $(x_{i-1}+x_{i+1}) > 2x_i$ and consequently
	\begin{equation}
		D^2(x_i) > 0.
	\end{equation}
\end{proof}
\begin{lemma}
	\begin{equation}
	D(H(D(x_i)))=-1 \Rightarrow D^2(x_i)<0 
	\end{equation}
\end{lemma}
\begin{proof}
	The proof is quite similar to that of Lemma 1, only exchanging the order of the inequalities.
\end{proof}

The case $D(H(D(x_i)))=0$ is not so straightforward, but we still have information concerning the monotonicity: $D(H(D(x_i)))=0$ implies that the function is an increasing or decreasing function at that point. For the digital image, this means absence of edges in that region.

When the application of the operator $L$ is repeated $n$ times, complex patterns caused by the non-linear characteristic of $H$ function arise, but we always have some relation with the operator $D^n$ and monotonicity of $D^{n-1}$. Works like \citep{L98} show how higher-order derivatives can be very helpful in feature detection. It essentially represents the fundamental concept of ``edge'' taken at different levels of abstraction. Edges are known since the earliest days of computer vision to be of key importance in this area and more recent approaches like the histograms of oriented gradients \citep{DT05} confirm their suitability for feature detection and recognition. The successive application and consequent abstraction yields a process also somewhat similar to what takes place in convolutional networks \citep{CMKV16} (even though in that case the operator is learned while here it is previously defined by the local binary pattern).

\section{Experiments}\label{sec:experiments}

The performance of the proposed CATex descriptors was assessed on the classification of three well-established texture databases, namely, KTH-TIPS2b \citep{HCFE04}, UIUC \citep{LSP05} and UMD \citep{XJF09}.

KTH-TIPS2b is a collection of color material pictures containing a total of 4752 images with resolution of $200\times 200$ pixels, evenly divided into 11 material classes. Each class is subdivided into 4 samples (`a', `b', `c', and `d') and each sample corresponds to specific settings of pose, scale and illumination. There are two big challenges in this data set: 1) focus on material rather than on photographed sample, i.e., the same material photographed at severely different conditions still should be assigned to the same class, resulting in high intra-class variance and impairing the performance of the texture descriptor; 2) the training/testing split scheme where one sample is selected for training and the other ones for testing, such that the algorithm should be able to recognize the material based on patterns of samples acquired under different conditions, such as illumination, scale, perspective, etc.

UIUC comprises a set of gray-level images ($256\times 256$ pixel resolution) divided into 25 classes, with 40 images in each class and acquired under uncontrolled/non-standardized conditions, with great variation in perspective, albedo, scale, and illumination. The training/testing split adopted here follows the classical protocol in the literature, corresponding to half of the samples randomly selected for training and half for testing. Such random division is repeated 10 times to provide statistical measures of average accuracy and standard deviation.

UMD shares similarities with UIUC, like the number of classes and images per class and the acquisition process under uncontrolled conditions. However the images in this case have higher resolution of $1280\times 960$ pixels. The testing/training protocol is also the same one adopted for UIUC.

\section{Results and Discussion}

The identification of image classes in each database was accomplished by using the proposed CATex descriptors as input to a machine learning classifier. Here we verify the performance of three well known classifiers in the literature, to know, support vector machines (SVM) \citep{CV95}, linear discriminant analysis (LDA) \citep{M04} and random forests (RF) \citep{H95}. 

For SVM we used the strategy ``one-against-all'' to allow its application to a multi-class problem. We also employed linear kernels, as more complex solutions have not provided significantly higher accuracy. The most impacting hyper-parameters were found out to be the decision margin $C$ in SVM and the number of decision trees in RF.  $C$ was determined by a grid search over a 5-fold split of the training set. The number of trees was obtained by setting a maximum value of $500$ and evaluating the ``error out of bag'', using the number of trees sufficient to stabilize that error. Given the high dimensionality of the proposed feature vector, we also applied a preprocessing stage of principal component analysis \citep{J86}. The number of principal components was also determined from the 5-fold nested cross-validation using only training samples. For the proposed method, we used $\alpha=0.10$ in this test. Table \ref{tab:classif} shows the average accuracy after using a number of rounds as specified for each database in Section \ref{sec:experiments}. 

LDA provided the best result in all cases. The good performance of LDA was expected as the combination of canonical correlation analysis (supervised) with Bayes theorem makes it more appropriate to work with multi-class problems with reasonable linear separability, i.e., when the descriptors work as expected \citep{M04}. Based on these results, we used LDA for all the compared data sets.
\begin{table}[!htpb]
	\centering
	\caption{Classification accuracy using SVM, LDA, and RF classifier.}
	\label{tab:classif}
	\begin{tabular}{cccc}
		\hline
			Database & SVM & LDA & RF\\	
		\hline
			KTH-TIPS2b & 60.0$\pm$3.5 & 66.7$\pm$4.0 & 64.5$\pm$1.8\\
			UIUC & 95.7$\pm$0.9 & 98.3$\pm$0.8 & 96.5$\pm$0.7\\
			UMD & 98.3$\pm$0.9 & 99.4$\pm$0.7 & 98.5$\pm$0.6\\
		\hline
	\end{tabular}
\end{table}

The most important hyperparameter in the proposed method is $\alpha$, as it controls the activation of the transition function. In practice, it controls the ``chaoticity'' of the entire process: the higher the value of $\alpha$ the more chaotic is the CA evolution. Figure \ref{fig:alpha1} exhibits the accuracy on the benchmark datasets when using $\alpha=0.05$, $\alpha=0.10$ and $\alpha=0.25$. In general, there is no remarkable discrepancy between the accuracies for different values. For consistency, we employed $\alpha=0.10$ for the next tests.
\begin{figure}[!htpb]
	\centering
	\includegraphics[width=.7\textwidth]{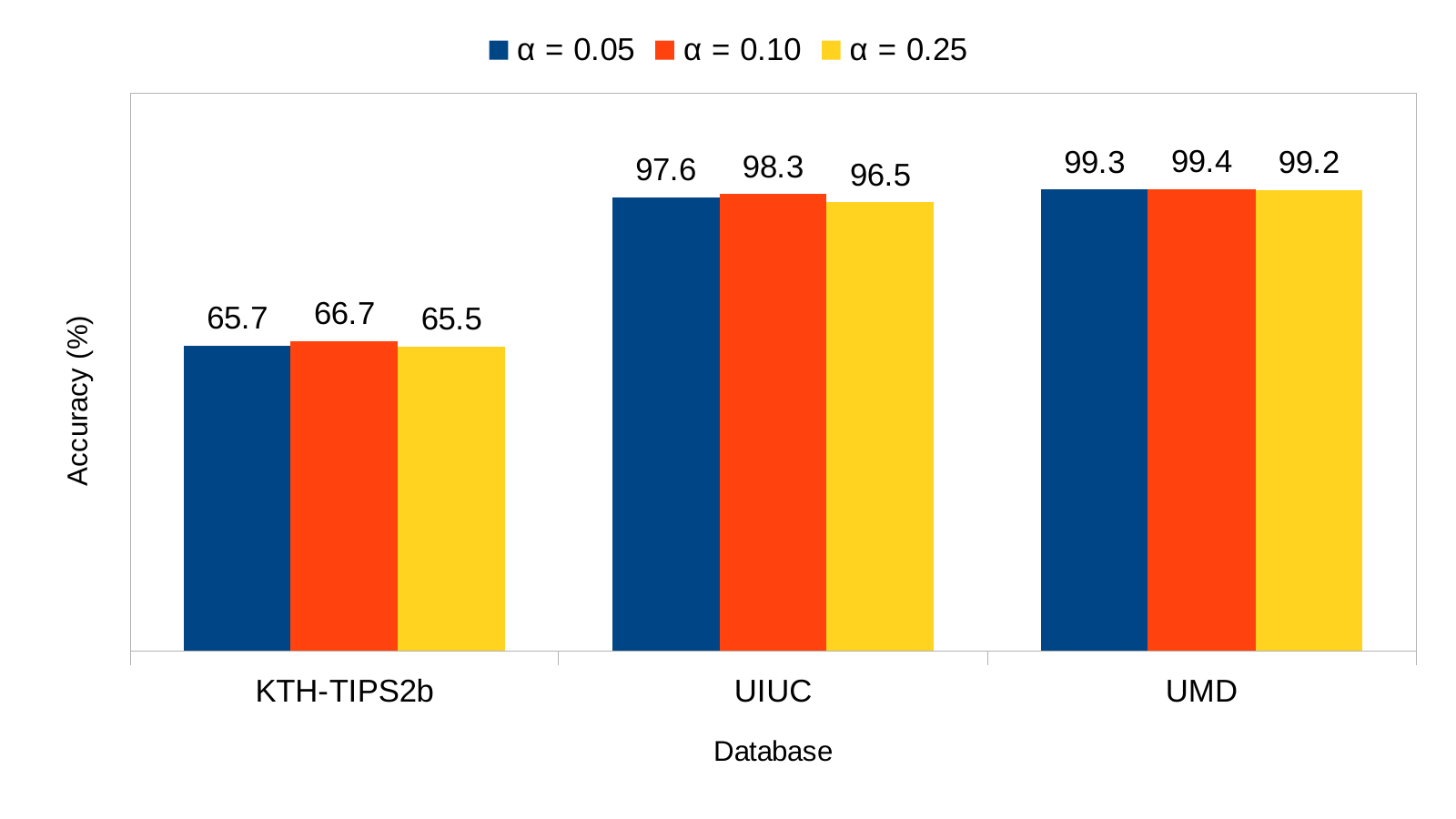}
	\caption{Accuracies in KTH-TIPS2b, UIUC and UMD using $\alpha$ values $0.05$, $0.10$ and $0.25$.}
	\label{fig:alpha1}
\end{figure}

Table \ref{tab:SRdatabase} lists the accuracy of CATex descriptors on the benchmark data sets, in comparison with other texture recognition approaches in the literature. The proposed descriptors achieved accuracy close to 100\% in UIUC and UMD, as expected from previously reported results achieved by state-of-the-art approaches in this task. Furthermore, it also reached competitive results in KTH-TIPS2b. The classification of this database is a significantly more challenging problem, especially considering the adopted protocol for training/testing that forces the algorithm to recognize materials based on training samples acquired under quite different conditions. Here it is also important to note that CATex does not take into account color information, which adds important discriminative features, but even under these circumstances, CATex still achieves good accuracy.
\begin{table}[!htpb]
	\centering
	\caption{Accuracy of the proposed descriptors compared with other texture descriptors in the literature. All the results except for the proposed CATex were obtained from the literature. A `-' indicates that no result was published for that method on that database.}
	\label{tab:SRdatabase}
	\begin{tabular}{cccc}
		Method & KTH-TIPS2b & UIUC & UMD\\
		\hline
		VZ-MR8 \citep{VZ05} & 46.3 & 92.9 & - \\
		LBP \citep{OPM02} & 50.5 & 88.4 & 96.1\\
		VZ-Joint \citep{VZ09} & 53.3 & 78.4 & - \\
		LBP-FH \citep{AMHP09} & 54.6 & - & - \\
		CLBP \citep{GZZ10} & 57.3 & 95.7 & 98.6 \\
		ELBP \citep{LZLKF12} & 58.1 & - & - \\
		SIFT + KCB \citep{CMKMV14} & 58.3 & 91.4 & 98.0\\
		SIFT + BoVW \citep{CMKMV14} & 58.4 & 96.1 & 98.1\\
		SIFT + VLAD \citep{CMKMV14} & 63.1 & 96.5 & 99.3\\
		RandNet (NNC) \citep{CJGLZM15} & 60.7\footnotemark[1] & 56.6 & 90.9 \\
		PCANet (NNC) \citep{CJGLZM15} & 59.4\footnotemark[1] & 57.7 & 90.5 \\
		BSIF \citep{KR12} & 54.3 & 73.4 & 96.1 \\			
		LBP$_{riu2}$/VAR \citep{OPM02} & 58.5\footnotemark[1] & 84.4 & 95.9 \\
		ScatNet (NNC) \citep{BM13} & 63.7\footnotemark[1] & 88.6 & 93.4 \\
		MRS4 \citep{VZ09} & - & 90.3 & - \\
		FC-CNN AlexNet \citep{CMKV16} & 71.5 & 91.1 & 95.9\\
		MFS \citep{XJF09} & - & 92.7 & 93.9 \\
		DeCAF \citep{CMKMV14} & 70.7 & 94.2 & 96.4\\
		FC-CNN VGGM \citep{CMKV16} & 71.0 & 94.5 & 97.2\\			
		Scattering\footnote{multiscale train} \citep{SM13} & - & 99.4 & 96.6\\
		(H+L)(S+R) \citep{LSP05} & - & 97.0 & 97.0\\			
		SIFT+LLC \citep{CMKV16} & 57.6 & 96.3 & 98.4\\
		WMFS \citep{XYLJ10} & - & 98.6 & 98.7\\			
		OTF \citep{XHJF12} & - & 98.1 & 98.8\\
		PLS \citep{QXSL14} & - & 96.6 & 99.0\\			
		\hline
		CATex (Proposed) & 66.7 & 98.3 & 99.4\\
		\hline		
	\end{tabular}
\end{table}

Figure \ref{fig:CM1} presents another important information concerning a classification task, which is the confusion matrix. As expected from the accuracies in Table \ref{tab:SRdatabase}, the diagonals of UIUC and UMD matrices are almost perfectly black, which corresponds to an accuracy close to 100\%. KTH-TIPS2b, on the other hand, has several gray squares outside the diagonal, corresponding to misclassified samples. Particularly, class 5 (``cotton'') is highly confused with 11 (``wool''), class 3 (``corduroy'') with 7 (``lettuce leaf''), and class 10 (``wood'') with 6 (``cracker''). The first case is not surprising as both groups contain types of fabric. The other ones are obviously different materials, but they are also composed by similar patterns. Color attribute could potentially be helpful for the material identification in these cases.
\begin{figure}[!htpb]
	\begin{tabular}{c}
		\begin{tabular}{cc}
			\includegraphics[width=.45\textwidth]{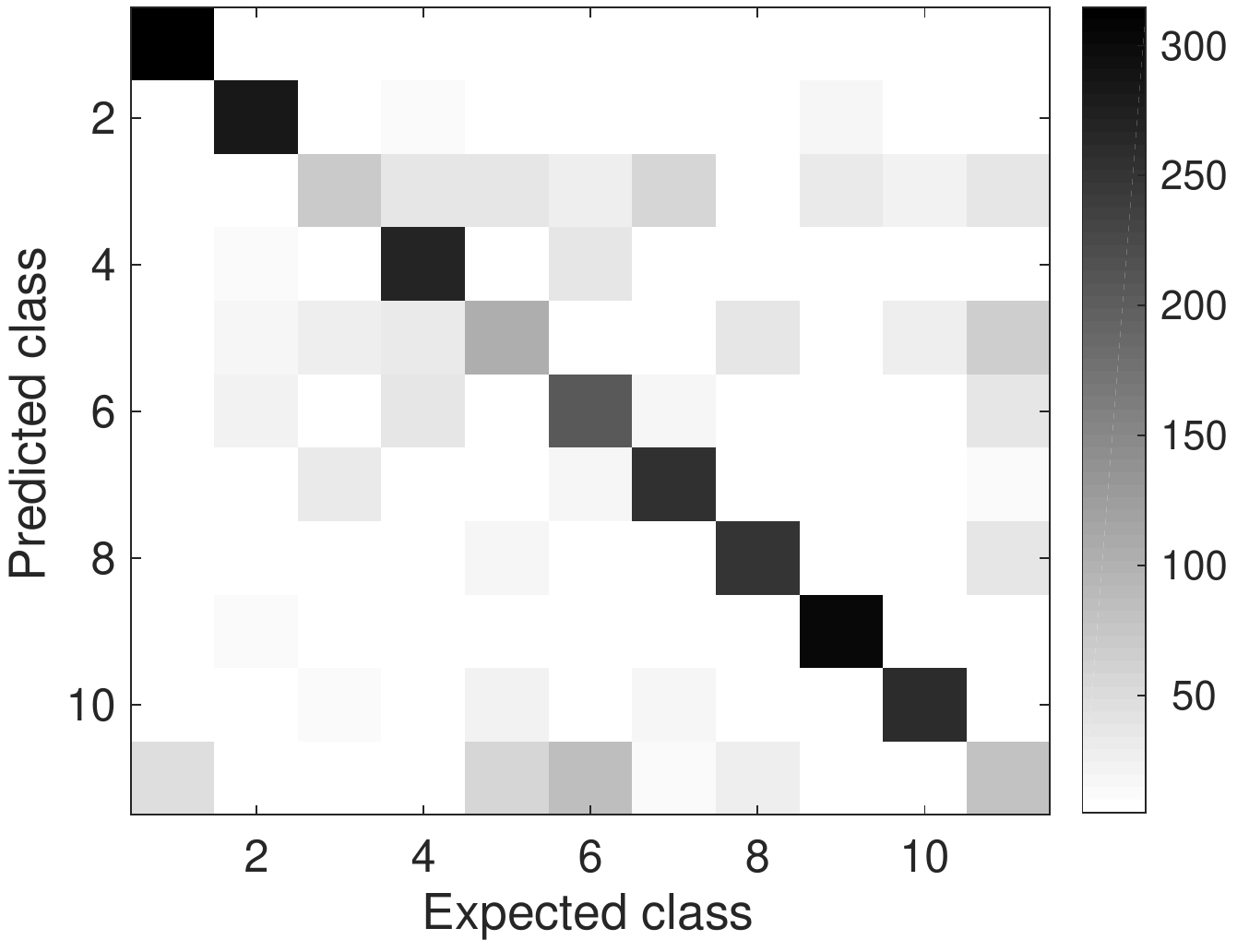} &
			\includegraphics[width=.45\textwidth]{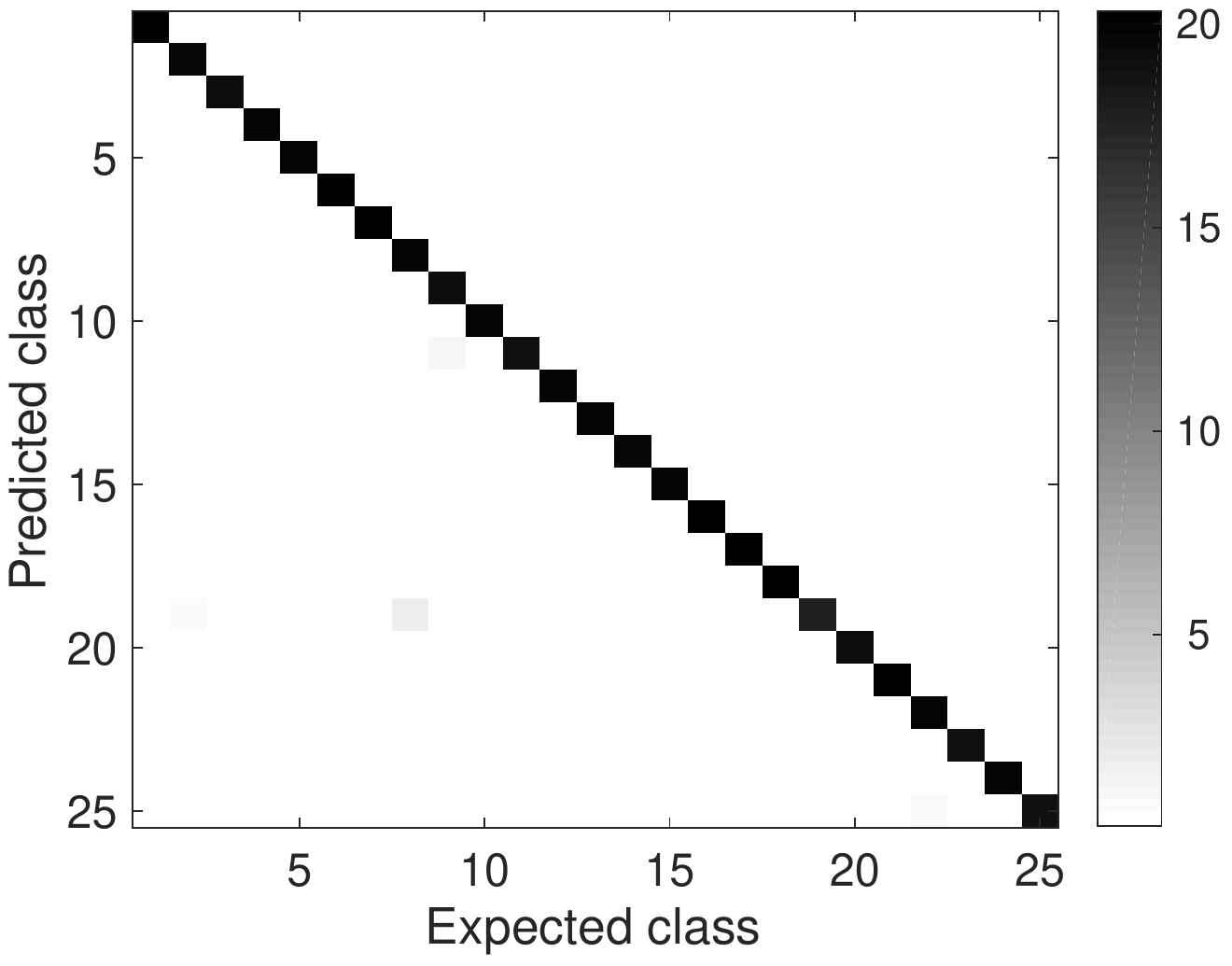}\\
			(a) & (b)\\
		\end{tabular}\\
		\begin{tabular}{c}
			\includegraphics[width=.45\textwidth]{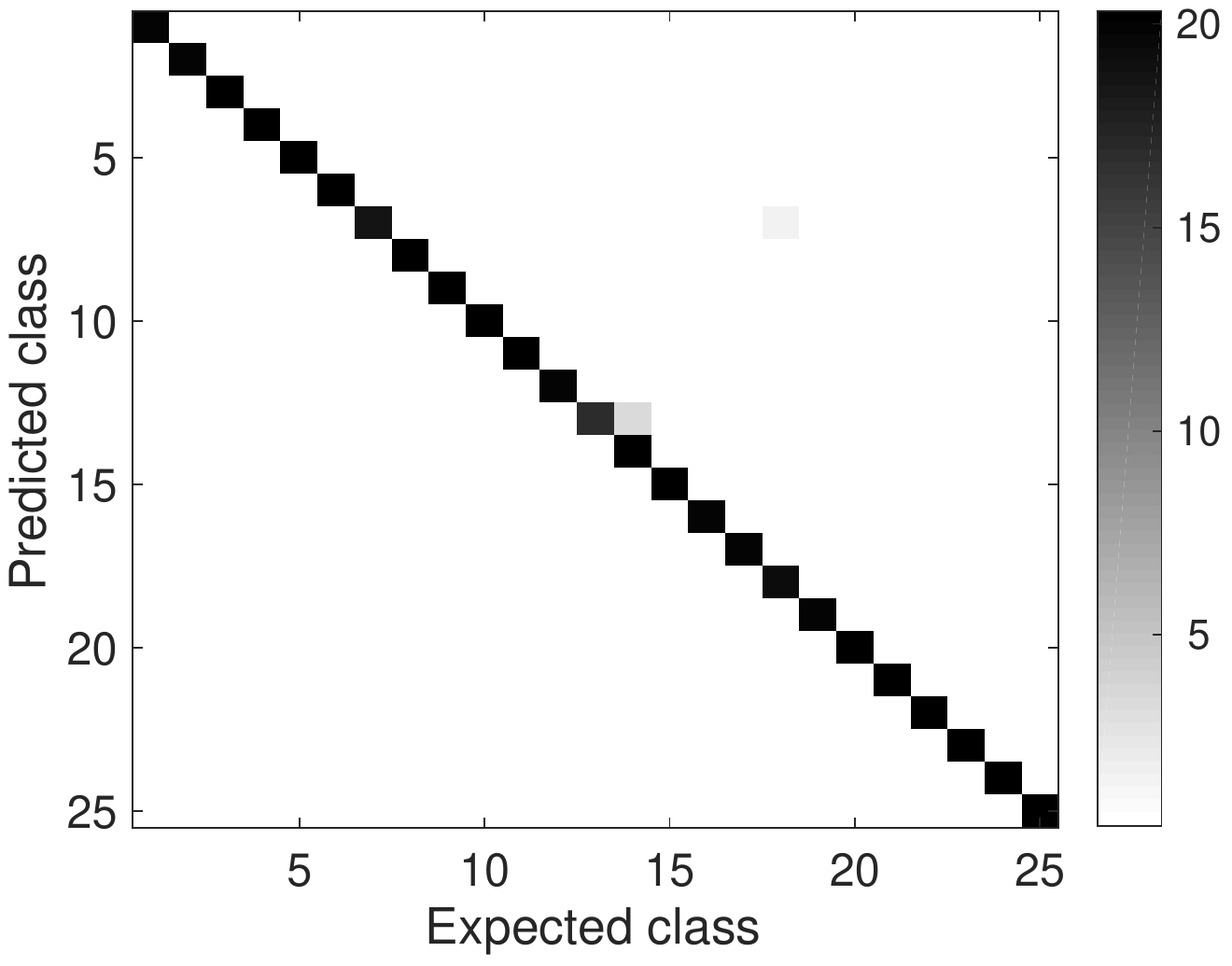}\\
			(c)\\
		\end{tabular}
	\end{tabular}
	\caption{Confusion matrices. (a) KTH-TIPS2b. (b) UIUC. (c) UMD.}
	\label{fig:CM1}
\end{figure}

The obtained results confirmed the expectations about the performance of CATex descriptors as a highly discriminative method as theoretically established in Section \ref{sec:motivation}. Reaffirming trends recently presented in the literature \citep{BM13,F20}, on the effectiveness of successive applications of nonlinear operator for texture recognition, the cascade application of nonlinear local operators has demons\-trated to be a powerful strategy for texture description. The discriminative features of the analyzed images are captured at different levels of abstraction, i.e., from micro to macro-texture regions. This mechanism somehow reminds what happens, for example, in the forward action of convolutional neural networks and explains the success of those networks in image recognition to a large extent. Such multi-level abstraction analysis is also the responsible here for the achieved success even in more complex tasks and makes the designed model an adequate framework for texture analysis.

\subsection{Application to the Identification of Plant Species}

To illustrate the effectiveness of the proposed methodology in an applied problem, we evaluate the CATex descriptors in the identification of species of Brazilian plants using images scanned from the leaf surface. The database, named 1200Tex \citep{CMB09}, comprises 1200 images, divided into 20 classes (plant species). 20 samples are collected \textit{in vivo} from each species, washed and aligned with the basal/apical axis. The photograph of each sample is split into 3 non-overlapping windows with resolution $128 \times 128$, totalizing 60 images per species. The split of training and testing set is identical to that one adopted for UIUC and UMD, i.e., half images randomly selected for training and the remaining half images for testing. Such process is repeated 10 times to provide average accuracy and deviation.

Figure \ref{fig:alpha2} shows the classification performance for different values of $\alpha$ hyperparameter. Like before, the difference is not so remarkable, although some advantage can be noticed for $\alpha=0.05$.
\begin{figure}[!htpb]
	\centering
	\includegraphics[width=.7\textwidth]{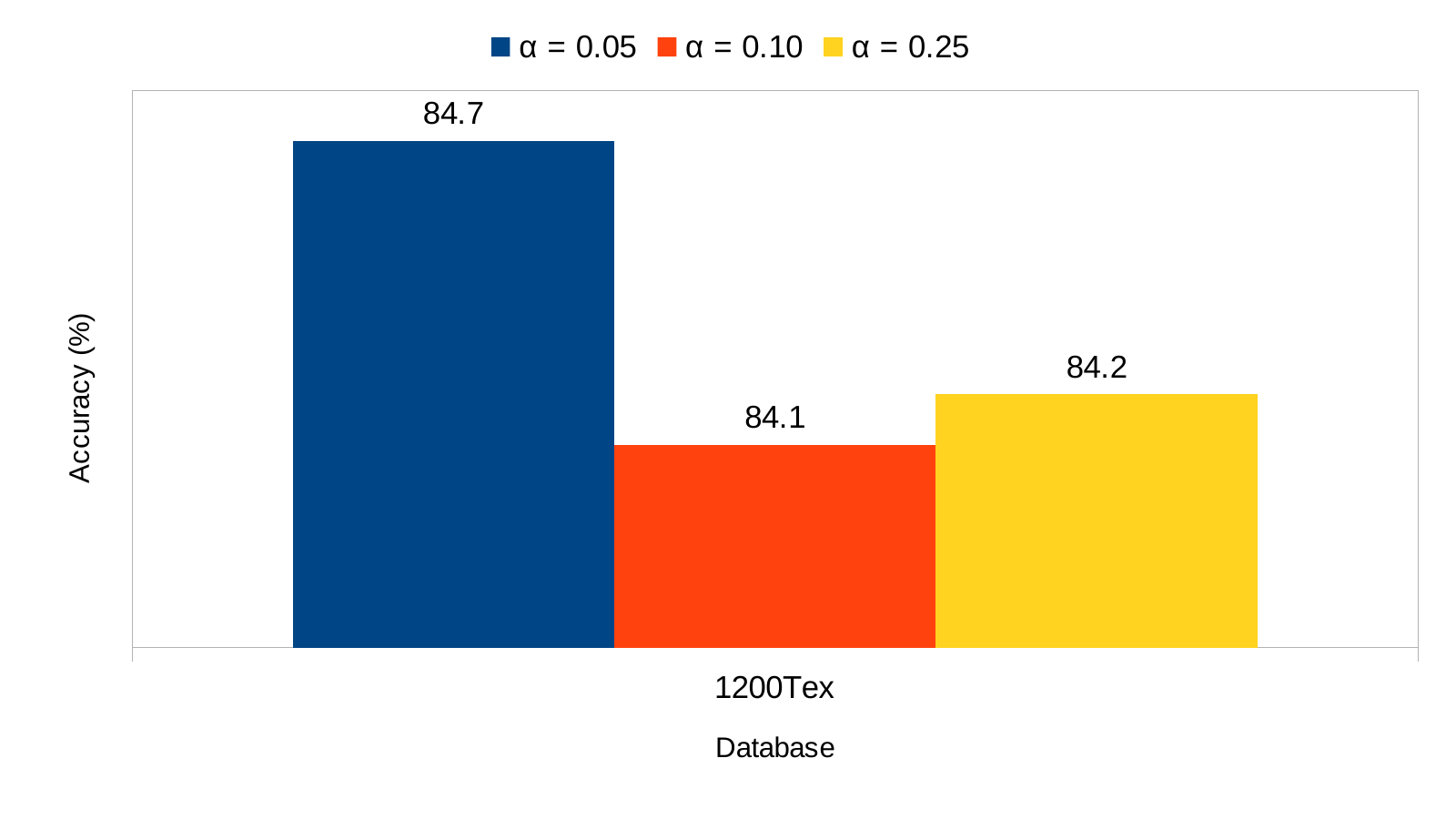}
	\caption{Accuracies in 1200Tex using $\alpha$ values $0.05$, $0.10$ and $0.25$.}
	\label{fig:alpha2}
\end{figure}

Table \ref{tab:SRdatabase_plant} shows the accuracy for CATex and other results on this database recently published in the literature. Similarly to what was observed in the benchmark databases, the proposed descriptors is competitive in the discrimination of plant species. The success in this type of task gives some evidences about the potential of our approach in a practical task. In fact, the identification of plant species using leaf images is a challenging test for any texture descriptor, considering the high inter-group similarity, which makes it a pro\-blem whose solution is nearly impossible for the human eyes. The values depicted in Table \ref{tab:SRdatabase_plant} confirm that the successive extraction of local features following a CA framework constitutes a suitable model to represent abstract features of the image that classical descriptors cannot access in their representation. 
\begin{table}[!htpb]
	\centering
	\caption{Accuracy of the proposed CATex descriptors compared with other results in the identification of plant species.}
	\label{tab:SRdatabase_plant}	
	\begin{tabular}{cc}
		\hline
		Method & Accuracy (\%)\\
		\hline
		LBPV \citep{GZZ10} & 70.8\\
		Network diffusion \citep{GSFB16} & 75.8\\
		FC-CNN VGGM \citep{CMKV16} & 78.0\\		
		Gabor \citep{CMB09} & 84.0\\
		FC-CNN VGGVD \citep{CMKV16} & 84.2\\
		\hline
		CATex (Proposed) & 84.7\\
		\hline
	\end{tabular}
\end{table}

Finally, Figure \ref{fig:CM2} depicts the confusion matrix for the proposed descriptors in 1200Tex database. The most difficult classes were 5 and 7, which were confused with each other. Figure \ref{fig:conf1200} illustrates a few samples from those groups. The images are pretty similar, sharing common homogeneous patterns. Nervure patterns are important features do distinguish plant species using leaves and those structures cannot be identified in the images of these groups.
\begin{figure}[!htpb]
	\centering
	\includegraphics[width=.7\textwidth]{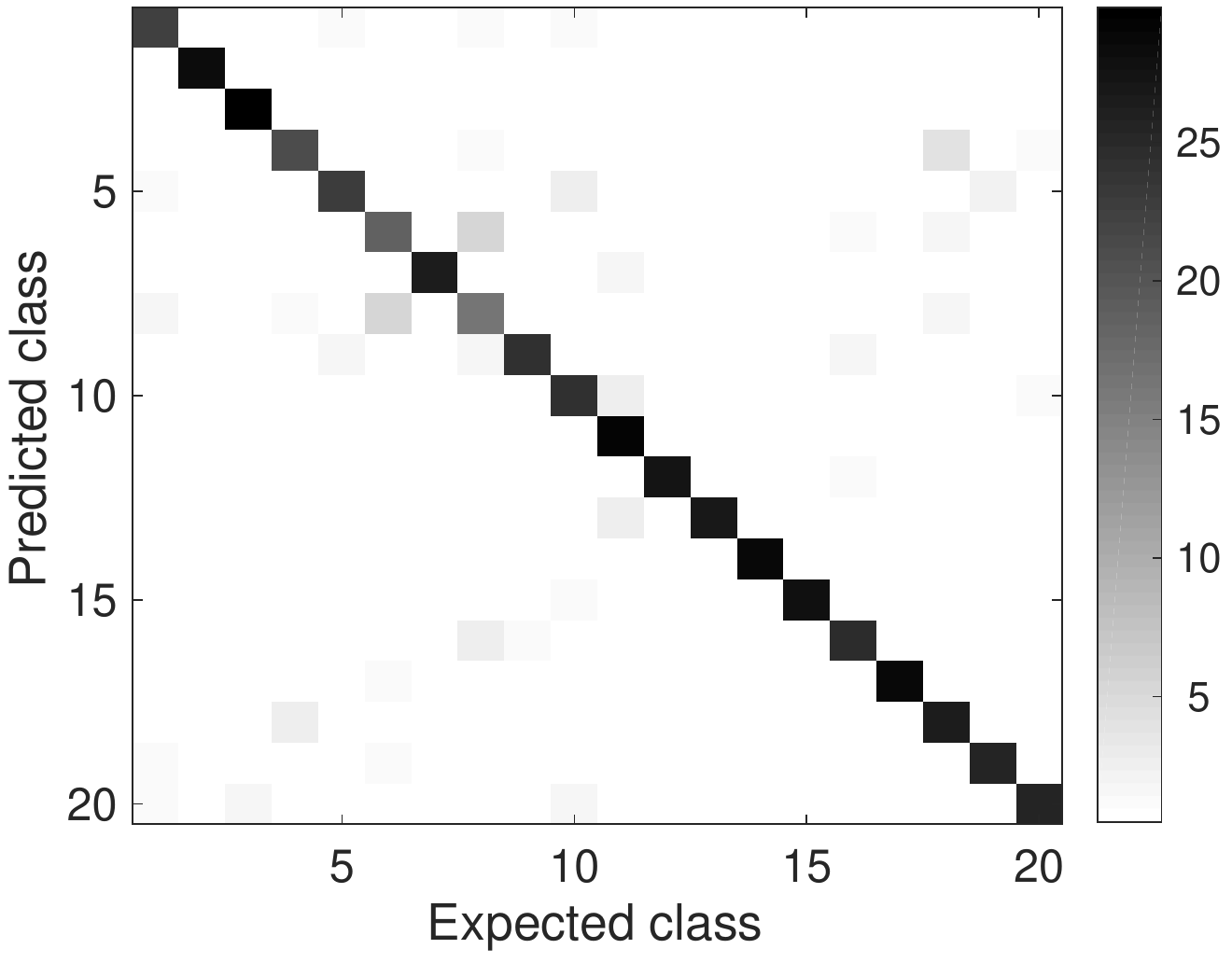}
	\caption{Confusion matrix for the plant database.}
	\label{fig:CM2}
\end{figure}
\begin{figure}[!htpb]
	\centering
	\includegraphics[width=.7\textwidth]{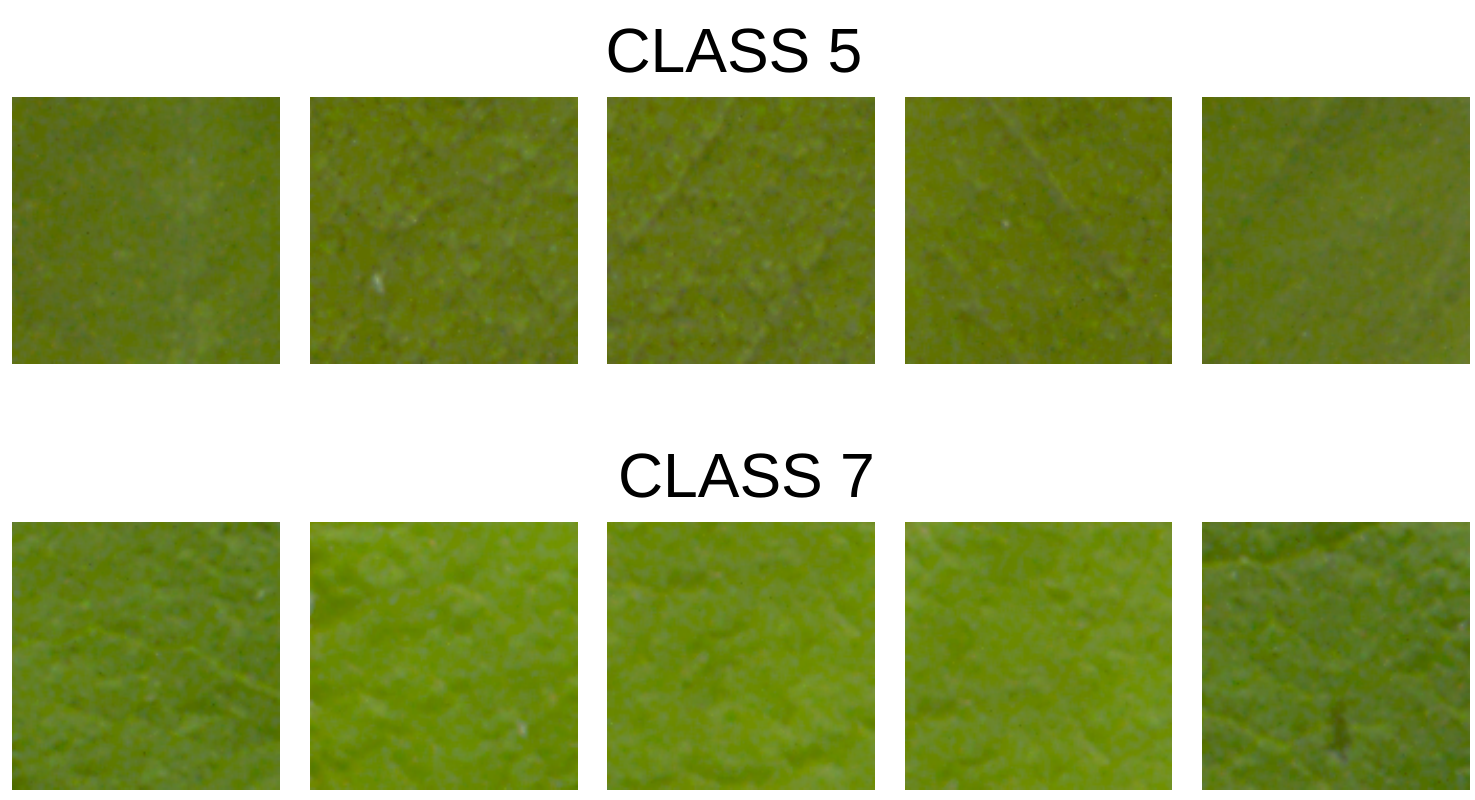}
	\caption{Classes most frequently confused in 1200Tex.}
	\label{fig:conf1200}
\end{figure}

To summarize, the results attest that the proposed method can also be safely applied to a real world problem, here the recognition of plant species based on leaf images. The accuracy close to 85\% is quite relevant if one takes into account the complexity of the task. Whereas in the benchmark database a human user, not necessarily specialist, could recognize the classes, discriminating species u\-sing leaves is rather difficult even for a botany specialist. The achieved accuracy confirms once more how much the machine algorithm can help the user in such types of tasks. Besides, we can also see that a ``handcrafted'' descriptor can be competitive with deep learning approaches like, for example, the FC-CNN VGGM and FC-CNN VGGVD methods in Table \ref{tab:SRdatabase_plant}.

\section{Conclusions}

This work proposed and investigated the performance of CATex texture descriptors, obtained from a cellular automata model in which the transition function is characterized by a local texture feature, derived from the well known local binary patterns.

The effectiveness of CATex features were assessed both on the classification of benchmark databases and on a practical problem, that of identifying species of Brazilian plants based on the texture of the leaf surface. In either situations, CATex confirmed its potential as powerful and robust descriptors, outperforming several state-of-the art approaches in terms of classification accuracy.

The developed methodology relies on the already known ability of recursive nonlinear operators in identifying complex patters in digital images. This characteristic allowed for a precise and robust texture description, as attested by the performance on databases like KTH-TIPS2b, which is highly affected by different types of intra-class variances. In summary, such robustness is also a favorable point to consider the proposed method as a natural candidate to model the numerous real-world problems involving texture description at some point.

\section*{Acknowledgements}

This work was supported by the Serrapilheira Institute (grant number Serra-1812-26426). J. B. Florindo also gratefully acknowledges the financial support from National Council for Scientific and Technological Development, Brazil (CNPq) (Grants \#301480/2016-8 and \#423292/2018-8).

\section*{References}


\end{document}